\documentclass[10pt,twocolumn,letterpaper]{article}

\usepackage{iccv}              %

\usepackage{bm}
\usepackage{tabularx}
\usepackage{graphicx}
\usepackage{pifont}
\usepackage{tcolorbox}
\usepackage[accsupp]{axessibility}

\usepackage{amsmath, amssymb, amsthm}

\newtheorem{theorem}{Theorem}

\newlength\savewidth

\usepackage{xspace}
\usepackage{array}
\usepackage[table, dvipsnames]{xcolor}
\usepackage{multirow}
\usepackage{makecell} %
\makeatletter
\DeclareRobustCommand\onedot{\futurelet\@let@token\@onedot}
\def\@onedot{\ifx\@let@token.\else.\null\fi\xspace}

\def\eg{\emph{e.g}\onedot} 
\def\ie{\emph{i.e}\onedot}

\makeatother

\newcolumntype{S}{>{\centering\arraybackslash}m{0.9cm}}
\newcolumntype{M}{>{\centering\arraybackslash}m{1.2cm}}
\newcolumntype{L}{>{\centering\arraybackslash}m{1.4cm}}
\definecolor{mygray}{gray}{.95}
\definecolor{mylightergray}{gray}{.99}
\definecolor{mygreen}{RGB}{10, 179, 33}
\usepackage{caption} 
\makeatletter
\newcommand{\thickhline}{%
    \noalign {\ifnum 0=`}\fi \hrule height 1pt
    \futurelet \reserved@a \@xhline
}
\newcolumntype{"}{@{\vrule width 1pt}}

\definecolor{iccvblue}{rgb}{0.21,0.49,0.74}
\usepackage[pagebackref,breaklinks,colorlinks,allcolors=iccvblue]{hyperref}

\usepackage{booktabs}
\usepackage{acronym}

\makeatletter
\DeclareRobustCommand\onedot{\futurelet\@let@token\@onedot}
\def\@onedot{\ifx\@let@token.\else.\null\fi\xspace}
\def\eg{\emph{e.g}\onedot} 

\def\ie{\emph{i.e}\onedot}

\makeatother

\newcommand{\stddev}[1]{{\tiny $\pm$#1}}
\acrodef{sota}[SOTA]{State-of-the-Art}

\newcommand{\zu}{$\bm{z}_u$\xspace}
\newcommand{\ze}{$\bm{z}_e$\xspace}

\newcommand{\bmu}{$\bm{u}$\xspace}
\newcommand{\bme}{$\bm{e}$\xspace}

\newcommand{\au}{$\mathcal{U}$\xspace}
\newcommand{\ee}{$\mathcal{E}$\xspace}

\newcommand{\aup}{$\mathcal{U}(\cdot;\theta_u)$\xspace}
\newcommand{\eep}{$\mathcal{E}(\cdot;\theta_e)$\xspace}

\newcommand{\tu}{$\mathcal{T}_u(\cdot;\phi_u)$\xspace}
\newcommand{\te}{$\mathcal{T}_e(\cdot;\phi_e)$\xspace}

\newcommand{\blank}{\rule{0.3cm}{0.25mm}~}

\newcommand{\obs}[2]{
    \begin{tcolorbox}[
        colback=gray!5,            %
        colframe=gray!60!black,    %
        arc=6pt,                   %
        boxrule=1pt,               %
        left=8pt, right=8pt,       %
        top=6pt, bottom=6pt        %
    ]
    \textbf{Observation \##1:} #2
    \end{tcolorbox}
}

\def\paperID{6257} %
\def\confName{ICCV}
\def\confYear{2025}

\title{Embodied Representation Alignment with Mirror Neurons}

\vspace{-5mm}
\author{Wentao Zhu\textsuperscript{1,2} \quad Zhining Zhang\textsuperscript{1} 
\quad Yuwei Ren\textsuperscript{3} \quad Yin Huang\textsuperscript{3}
\quad Hao Xu\textsuperscript{3} \quad Yizhou Wang\textsuperscript{1,4} \\[1.2ex]
    \textsuperscript{1~}Center on Frontiers of Computing Studies, School of Compter Science, Peking University \\
    \textsuperscript{2~}Eastern Institute of Technology, Ningbo \quad
    \textsuperscript{3~}Qualcomm AI Research\\
    \textsuperscript{4~}Inst. for Artificial Intelligence, Peking University\\
}

\begin{document}
\maketitle
\vspace{-5mm}

\begin{abstract}
Mirror neurons are a class of neurons that activate both when an individual observes an action and when they perform the same action. This mechanism reveals a fundamental interplay between action understanding and embodied execution, suggesting that these two abilities are inherently connected. Nonetheless, existing machine learning methods largely overlook this interplay, treating these abilities as separate tasks. In this study, we provide a unified perspective in modeling them through the lens of representation learning. We first observe that their intermediate representations spontaneously align. Inspired by mirror neurons, we further introduce an approach that explicitly aligns the representations of observed and executed actions. Specifically, we employ two linear layers to map the representations to a shared latent space, where contrastive learning enforces the alignment of corresponding representations, effectively maximizing their mutual information. Experiments demonstrate that this simple approach fosters mutual synergy between the two tasks, effectively improving representation quality and generalization.
\end{abstract}

\section{Introduction}

\begin{flushright}
    \textit{``The body is our general medium for having a world.''} \\[0.5ex]
    --- Maurice Merleau-Ponty \\[0.5ex]
\end{flushright}

Neuroscience research has uncovered a fascinating mechanism behind multiple cognitive abilities: \emph{mirror neurons}. First identified in macaque monkeys, these neurons fire both when an individual observes an action and when they perform it themselves~\citep{di1992understanding,gallese1996action,rizzolatti2004mirror}. In essence, the neural representations of observed and executed actions are inherently \emph{aligned}. Subsequent studies confirm similar systems in the human brain, where observing others’ actions activates corresponding motor regions, as if the observer were performing them~\citep{fadiga1995motor,keysers2009expanding}. This suggests that action understanding arises from neural simulations of observed behaviors, not merely abstract reasoning~\citep{rizzolatti2008mirrors}. By mapping external movements onto its own motor repertoire, the brain internally ``experiences'' the action, enabling an intuitive grasp of others’ intentions~\citep{jeannerod2001neural}.

\begin{figure}[tb]
    \centering
    \includegraphics[width=\linewidth]{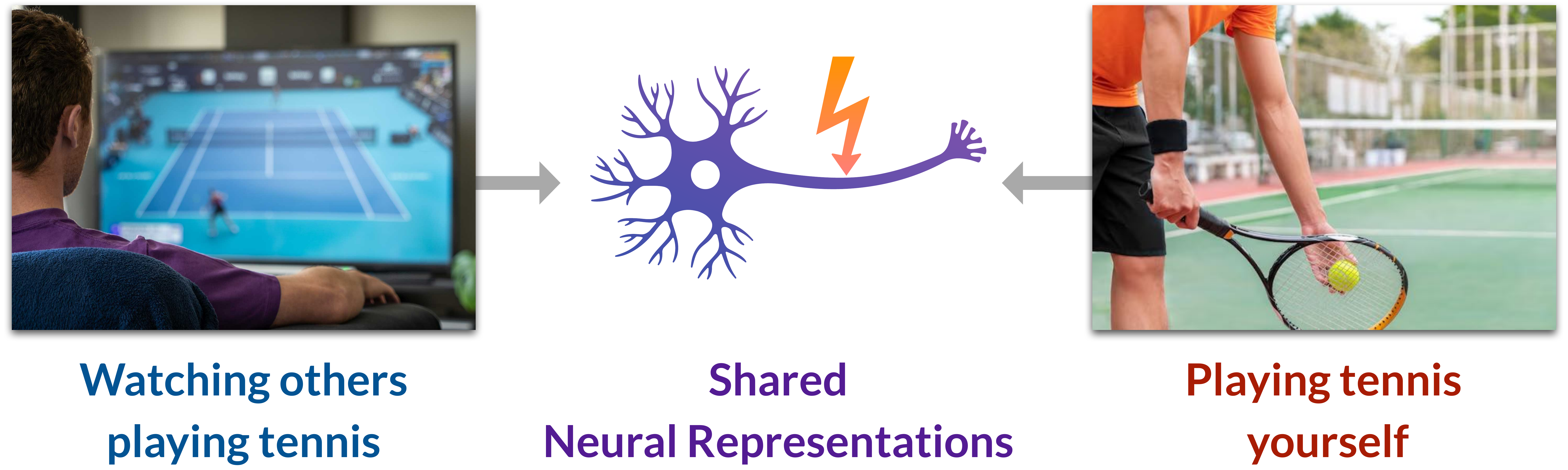}
\caption{\textbf{Conceptual demonstration of mirror neurons.} These neurons activate both when observing an action and when performing it oneself, illustrating the shared neural representations that link perceptual and motor systems.}
    \label{fig:conceptual}
        \vspace{-6mm}
\end{figure}

This neural mechanism highlights the tight bond between two fundamental cognitive abilities: action understanding and embodied execution. Concretely, action understanding enables agents to interpret the meaning and intent behind others' actions, while embodied execution allows them to physically interact with the environment to achieve goals.
Crucially, these two abilities are deeply interconnected: action understanding supports embodied execution by guiding imitation learning and skill acquisition, while in turn, embodied execution provides firsthand sensorimotor experience that refines and deepens action understanding~\cite{erlhagen2006dynamic,loucks2012role}.
The biological foundation of this interplay is exemplified by mirror neurons, which also demonstrates the core idea of \emph{embodied cognition}~\citep{clark1998being,gallese2005brain,barsalou2008grounded,varela2017embodied,sep-embodied-cognition}—the notion that cognitive processes are not merely functions of the brain or abstract activities of the mind, but are deeply rooted in the body’s sensorimotor interactions with the world.
Despite this biological synergy, current machine learning approaches typically address action understanding and embodied execution independently, overlooking their potential to inform and enhance each other~\citep{goyal2023rvt,wang2023internvid,motionbert2022}.
This separation impedes the learning of generalizable and comprehensive action representations, which in turn limits performance on downstream tasks.

In light of this, this work proposes to unify action understanding and embodied execution through the lens of representation learning.
We begin by examining the relationship between the neural representations learned by models for these two tasks. Specifically, we investigate whether the representations align when observing and executing the same actions, how this alignment evolves during training, and its correlation with task success (\cref{sec:probing}).
Building on these insights, we propose a paradigm that jointly trains both models, bridging the gap between action understanding and embodied execution. The core idea is to project the agent’s representations of observed and executed actions into a shared latent space, where contrastive learning enforces their alignment (\cref{sec:aligning}). This approach explicitly aligns the representations of observed and executed actions, which is exactly what biological mirror neurons do. From the viewpoint of information theory, it is equivalent to maximizing the mutual information between the neural representations of the same semantic actions in action understanding and embodied execution.
We evaluate our method on action recognition and multi-task object manipulation benchmarks, demonstrating that the proposed framework enables the two tasks to reinforce each other (\cref{sec:exp}). Our results show that the learned representations are more disentangled and robust, leading to improved generalization. Furthermore, we investigate the impact of different alignment strategies, exploring the appropriate granularity for alignment.
We summarize the key contributions of this paper below:
\begin{itemize}
    \item We conceptualize action understanding and embodied execution as a unified system, grounded in neuroscience and cognitive insights.  
    \item We discover that these two models exhibit representation alignment spontaneously, which correlates to some extent with task success.  
    \item We introduce a representation learning approach inspired by mirror neurons, which directly aligns the neural representations of observing and executing corresponding tasks, thereby effectively maximizing their mutual information.  
    \item Experiments demonstrate that our method is simple and effective, enhancing the generalization and representation learning quality for both tasks.
\end{itemize}

\section{Probing Representation Alignment}
\label{sec:probing}

First, we aim to investigate whether action understanding and embodied execution models, when trained separately, exhibit neural representation alignment similar to mirror neurons. If such alignment exists, we further seek to understand the factors that influence it.

\subsection{Problem Formulation}

Given an  action understanding model \aup and an embodied execution model \eep, we aim to extract their intermediate neural representations and study their alignment. Specifically, the action understanding model $\mathcal{U}$ takes a perceptual sequence $V$, such as a video, and produces an internal representation $\bm{u}$, from which it predicts the semantic label $\hat{y} = \mathcal{U}(V)$.
Meanwhile, the embodied execution model $\mathcal{E}$ takes the environment state $S$ and an instruction $I$ (\eg, a natural language command) as input, generating the representation $\bm{e}$, which is then decoded to the next action $\hat{a} = \mathcal{E}(S, I)$.

The neural representations of observed and executed actions, $\bm{u}$ and $\bm{e}$, originally reside in different high-dimensional representation spaces. In principle, they share some common information, such as the semantic concepts and spatial relationships of the scene. Meanwhile, they would also capture unique aspects specific to the task.
For instance, $\bm{u}$ may emphasize high-level action semantics and visual patterns essential for understanding, whereas $\bm{e}$ is more attuned to physical constraints and feasibility, which are crucial for execution. 

Therefore, to effectively measure their alignment, we aim to learn a pair of linear transformations, \tu and \te, that map $\bm{u}$ and $\bm{e}$ onto a shared latent space $\mathbb{Z}$ and align them. Formally:
\begin{equation}
    \bm{z}_u = \mathcal{T}_u(\bm{u}) \in \mathbb{Z}, \quad 
    \bm{z}_e = \mathcal{T}_e(\bm{e}) \in \mathbb{Z},
\end{equation}
where \zu and \ze denote the aligned action understanding and embodied execution representations, respectively.

\begin{figure}[tb]
    \centering
    \includegraphics[width=\linewidth]{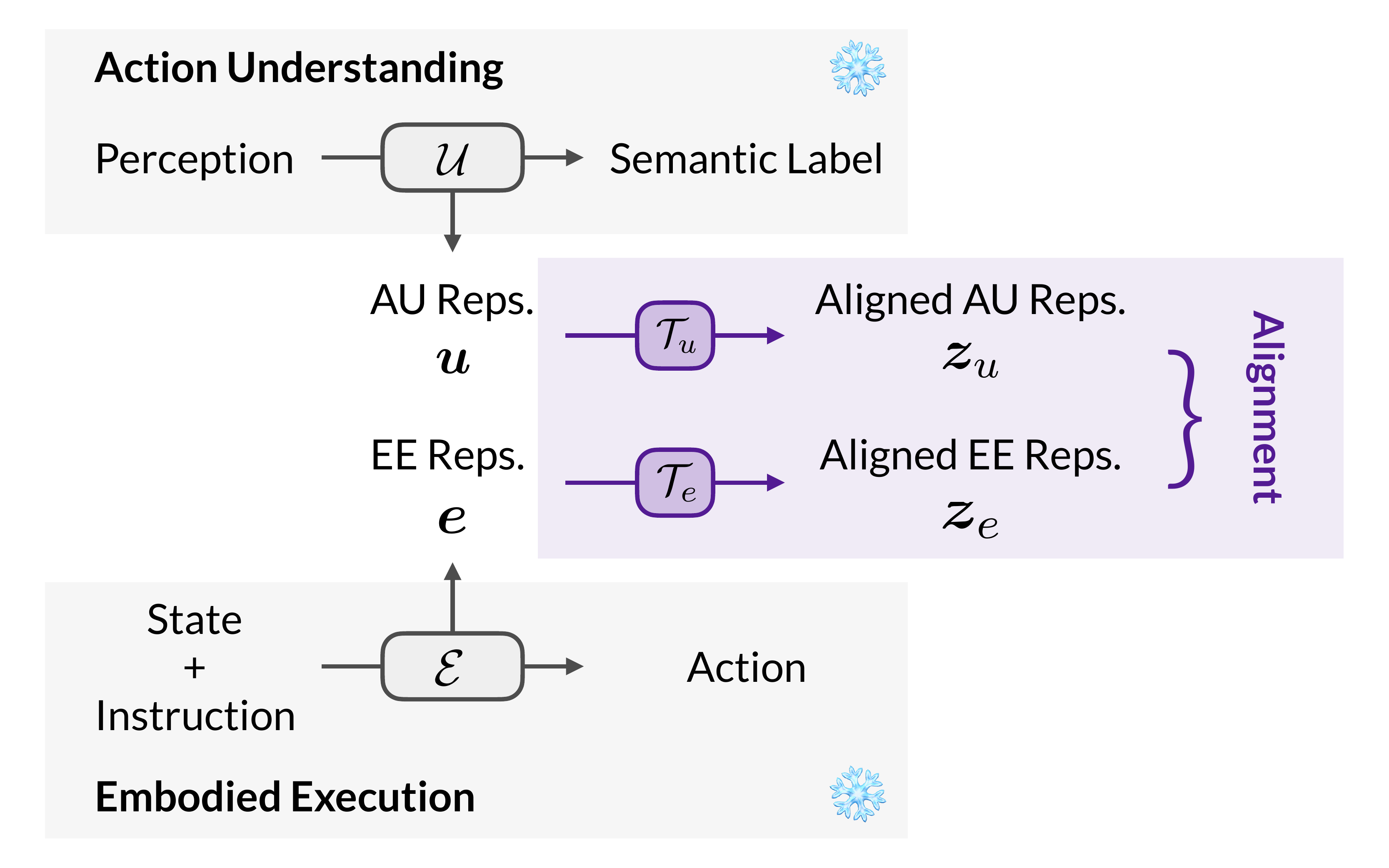}
    \vspace{-4ex}
    \caption{\textbf{Alignment probing of off-the-shelf model representations.} We extract internal neural representations from an Action Understanding (AU) model and an Embodied Execution (EE) model and train two linear transformations, $\mathcal{T}_u$ and $\mathcal{T}_e$, to align them. Both models are pretrained separately and remain frozen during this process.}
    \label{fig:probing}
    \vspace{-4mm}
\end{figure}

\subsection{Probing}

We employ alignment probing~\cite{zhang2024assessing} to align the representations and evaluate the alignment. \Cref{fig:probing} provides an overview of the pipeline.
Given off-the-shelf $\mathcal{U}(\cdot)$ and $\mathcal{E}(\cdot)$, we use contrastive learning to align the corresponding action pairs in the shared latent space $\mathbb{Z}$. Specifically, we optimize the following bidirectional InfoNCE loss:
\begin{equation}
\label{eq:align_loss}
    \begin{split}
        \mathcal{L}_\text{align} = - \frac{1}{2B} \sum_{i=1}^{B} 
        \Bigg[ 
            \log \frac{\exp(\text{sim}(\bm{z}_u^{(i)}, \bm{z}_e^{(i)}) / \tau)}
            {\sum\limits_{j=1}^{B} \exp(\text{sim}(\bm{z}_u^{(i)}, \bm{z}_e^{(j)}) / \tau)} \\
            + \log \frac{\exp(\text{sim}(\bm{z}_u^{(i)}, \bm{z}_e^{(i)}) / \tau)}
            {\sum\limits_{j=1}^{B} \exp(\text{sim}(\bm{z}_e^{(i)}, \bm{z}_u^{(j)}) / \tau)}
        \Bigg].
    \end{split}
\end{equation}
where $B$ is the batch size, and $\tau$ is a temperature scaling parameter. Each pair $(\bm{z}_u^{(i)}, \bm{z}_e^{(i)})$ consists of feature representations derived from observing and executing actions that share the same intent. The similarity function $\text{sim}(\cdot, \cdot)$ is defined as the cosine similarity:
\begin{equation}
    \text{sim}(\bm{z}_u, \bm{z}_e) = \frac{\bm{z}_u^\top \bm{z}_e}{\|\bm{z}_u\| \|\bm{z}_e\|}.
\end{equation}
We optimize the parameters of the linear transformations to minimize the alignment loss:
\begin{equation}
    (\phi_u^*, \phi_e^*) = \arg\min_{\phi_u, \phi_e} \mathcal{L}_\text{align}.
\end{equation}

From an information-theoretic perspective, this optimization objective is equivalent to optimizing a lower bound to estimate the mutual information between the action understanding representation \bmu and the embodied execution representation \bme. The theoretical derivation can be found in~\cref{sec:theoretical} and~\cref{thm_probe}.
In practice, we adopt ViCLIP~\cite{wang2023internvid}, a video-text representation learning model, as the action understanding model \au. We take the output feature of the video encoder as \bmu.  
Furthermore, we utilize ARP~\cite{zhang2024arp}, a pretrained language-conditioned robotic manipulation model, as the embodied execution model \ee. We take the output feature from the last block of the policy network (a chunking causal transformer) as \bme. 
For more implementation details, please refer to~\cref{exp_implementation} and~\cref{supp_implementation}. To measure the degree of alignment, we compute the average Recall@1 of bidirectional nearest neighbor retrieval on a held-out test set after training \tu and \te.

\subsection{Observations}

\begin{figure}[t]
    \centering
    \includegraphics[width=\linewidth]{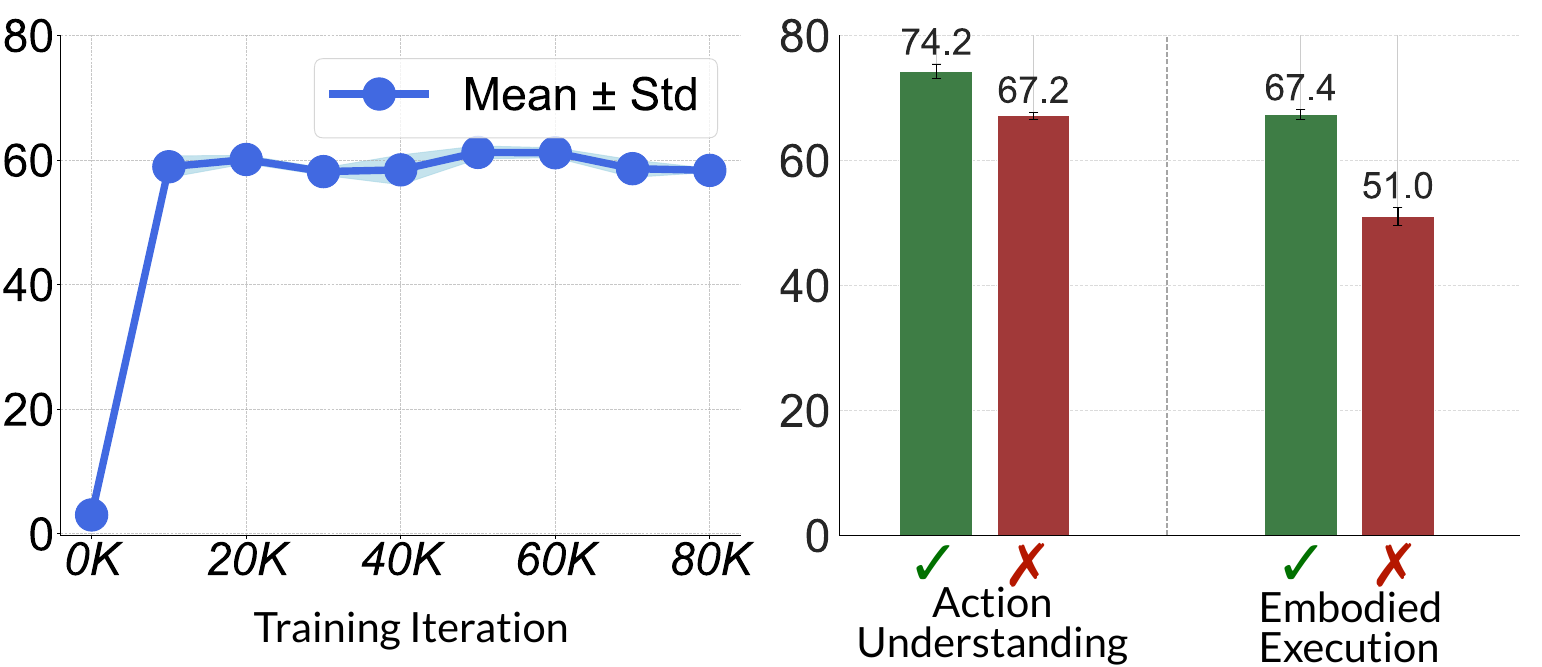}
    \caption{\textbf{Alignment probing results.}  
Left: Alignment scores tracked over the model training progress.  
Right: Alignment scores computed separately for task success and failure subsets from two pretrained models.  
We use the average retrieval accuracy (\%) on the test set as a measure of representation alignment.  
Alignment is trained and tested under four different settings while ensuring equal sample sizes.}
    \label{fig:nomn}
    \vspace{-2ex}
\end{figure}

We explore whether models of action understanding and embodied execution, though trained in isolation, converge spontaneously—like dancers finding harmony without rehearsal—in their neural representations.
To this end, we measure the representation alignment of the two models at different stages of the training process.  
\Cref{fig:nomn} (L) shows that the alignment between the two models increases rapidly in the early stages of training and reaches a high level. In particular, training only two linear transformations achieves more than 60\% accuracy in determining whether the representations correspond to the same underlying action.
This phenomenon may arise because the action understanding model and the embodied execution model, despite being trained for different tasks, both require effective modeling and abstraction of the underlying principles and structured patterns of object interactions.
This observation also aligns with the \emph{Platonic Representation Hypothesis}~\cite{huh2024position}, which states that neural networks, trained with different objectives on different datasets and modalities, converge to a shared statistical model of reality in their representation spaces.

 \obs{1}{
Independently trained action understanding and embodied execution models exhibit a swift emergence of meaningful neural alignment, suggesting a convergence towards representations of the common underlying reality.}

Furthermore, we aim to investigate whether there exists a relationship between the degree of representation alignment and task success. To this end, we conduct two sets of experiments:  
(1) We compute the representation alignment for correctly and incorrectly recognized samples in action understanding with respect to embodied execution.  
(2) Similarly, we compute the alignment for successfully and unsuccessfully completed tasks in embodied execution with respect to action understanding.  
In all experiments, we control for the same number of training and testing samples in alignment probing and train separate probes for each setting.

As shown in \Cref{fig:nomn} (R), we observe that the alignment score for the subset of task-successful samples is significantly higher than that for the subset of task-failed samples. We hypothesize that this may be because task-successful samples are associated with higher-quality representations of the underlying reality, leading to more robust alignment.

\obs{2}{
Task-successful samples exhibit significantly higher representation alignment, potentially due to their higher-quality representations of the underlying reality.}

This observation further motivates us to consider whether a causal relationship exists.
Specifically, is representation alignment merely a byproduct of better representations, or can it also actively contribute to their formation?
Could promoting neural representation alignment improve representation quality and, in turn, benefit the task performance of both models?

\section{Aligning Representations with Mirror Neurons}
\label{sec:aligning}

Inspired by the observations above, we further explore explicitly aligning the neural representations of action understanding and embodied execution during model training. 
Structurally, this approach establishes a shared representational space by directly linking perceptual and motor pathways, reflecting the anatomical organization of sensorimotor circuits in biological systems. Functionally, these structural principles promote bidirectional information flow, enabling seamless sensorimotor integration, akin to the role of mirror neurons in action recognition and execution.

To achieve this, we simultaneously train the action understanding model \aup, and the embodied execution model \eep, in a coupled structure, as shown in~\cref{fig:aligning}.
We adopt a minimal modification approach, where, on top of the original training objectives of both models, we only introduce two linear transformations \tu and \te to align their internal representations throughout this process.
Specifically, we denote the training objective of the action understanding model as $\mathcal{L}_\text{AU}(\theta_u)$ and that of the embodied execution model as $\mathcal{L}_\text{EE}(\theta_e)$. We define these two task-specific training objectives following previous work~\cite{wang2023internvid,zhang2024arp}, as they are not the primary focus of this study. 
Notably, we simply introduce a bidirectional contrastive loss $\mathcal{L}_\text{align}$, as defined in \cref{eq:align_loss}. 
The first term encourages each sample's representation in the action understanding model to align closely with its paired counterpart in the embodied execution model while distinguishing it from other samples in the batch. The second term symmetrically enforces the same constraint for the embodied execution model relative to the action understanding model. This bidirectional alignment objective helps regularize the latent spaces robustly, effectively bringing the representations of observing and executing similar actions closer while pushing apart those of different actions.

\begin{figure}[tb]
    \centering
    \includegraphics[width=\linewidth]{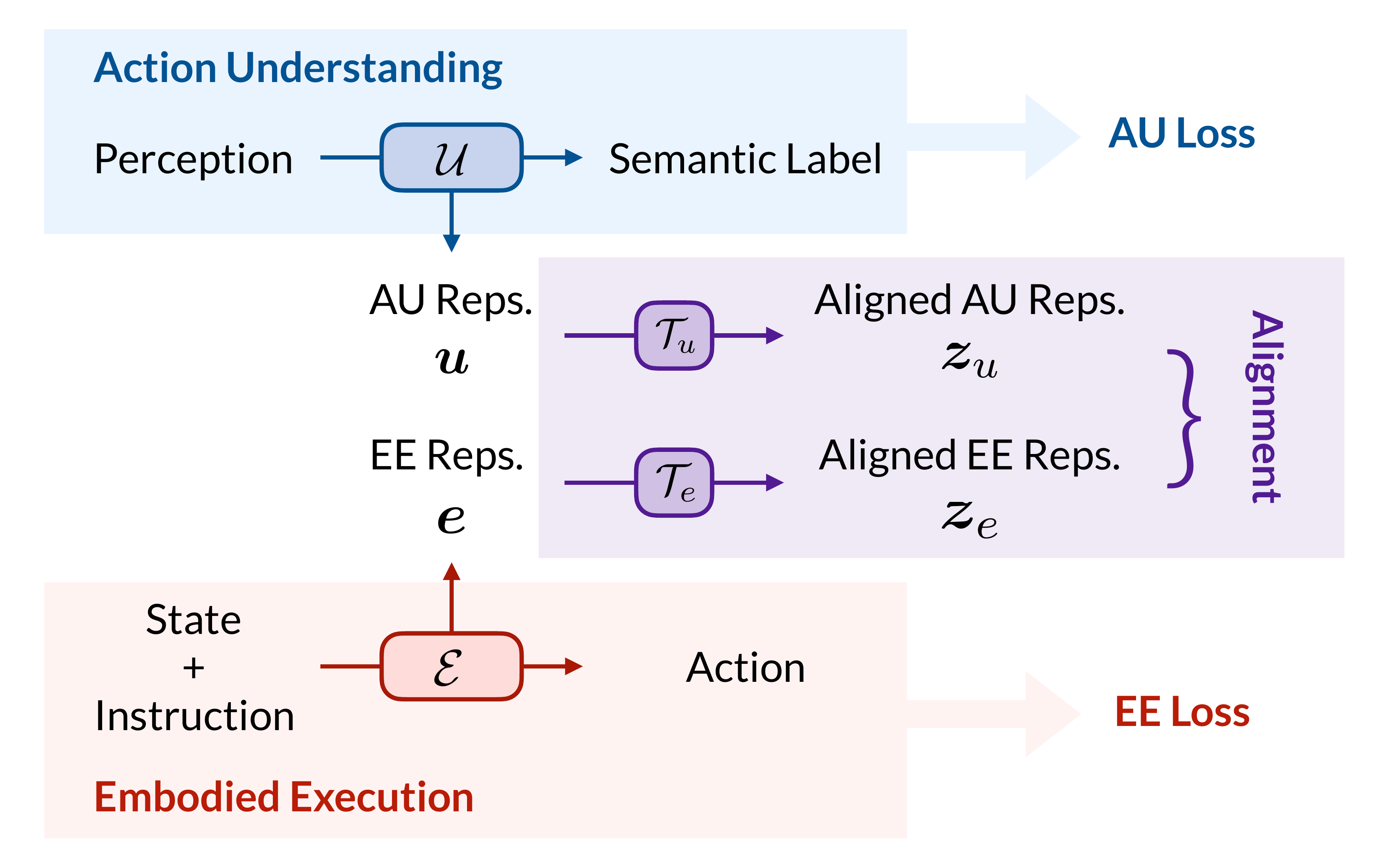}
\caption{\textbf{Aligning action understanding and embodied execution with mirror neurons.}  
We employ two linear layers to align intermediate representations between the Action Understanding (AU) model and the Embodied Execution (EE) model.  
The two models are jointly trained to align their representations while simultaneously optimizing for their respective task objectives.}
    \label{fig:aligning}
\end{figure}

Our final training objective can be formulated as:

\begin{equation}
    \mathcal{L}_\text{final} = \mathcal{L}_\text{AU} + \lambda_\text{EE} \mathcal{L}_\text{EE} + \lambda_\text{align} \mathcal{L}_\text{align},
\end{equation}
where $\lambda_\text{EE}$ and $\lambda_\text{align}$ are hyperparameters to balance the loss scales. We jointly optimize parameters of all models end-to-end:
\begin{equation}
    (\phi_u^*, \phi_e^*, \theta_u^*, \theta_e^*) = \arg\min_{\phi_u, \phi_e, \theta_u, \theta_e} \mathcal{L}_\text{final}.
\end{equation}
In this process, the optimization objective of representation alignment effectively regularizes the training of \aup and \eep simultaneously. This objective is consistent with the mechanism of mirror neurons, which map observed and executed similar actions to shared neural representations.
From an information-theoretic perspective, this optimization objective is equivalent to optimizing a lower bound to maximize the mutual information between \bmu and \bme. Please refer to~\cref{sec:theoretical} and~\cref{thm_align} for the theoretical derivation.

\section{Experiments}
\label{sec:exp}

In this section, we aim to investigate the impact of the proposed mirror framework on model training. Specifically, we design experiments to address the following questions: 
(1) How does it affect the performance of models for action understanding and embodied execution, respectively? 
(2) What influence does it have on the representations of these two tasks? 
(3) Which representations should be aligned, and what are the effects of this strategy?

\begin{table*}[t]
\center
\small
\setlength{\tabcolsep}{4pt} %
\resizebox{1.0\linewidth}{!}{
\begin{tabular}{l | c c c c c c c c c c c c c c c c c c | c}
\thickhline 
Method & \makecell{Close\\Jar} & \makecell{Drag\\Stick} & \makecell{Insert\\Peg} & \makecell{Meat off\\Grill} & \makecell{Open\\Drawer} & \makecell{Place\\Cup} & \makecell{Place\\Wine} & \makecell{Push\\Btn.} & \makecell{Put in\\Cpd.} & 
\makecell{Put in\\Drawer} & \makecell{Put in\\Safe}  & \makecell{Screw\\Bulb} & \makecell{Slide\\Block} & \makecell{Sort\\Shape} & \makecell{Stack\\Block} & \makecell{Stack\\Cup} & \makecell{Sweep\\Dust} & \makecell{Turn\\Tap} & Avg. \\
\hline 
ViCLIP~\cite{wang2023internvid}$\dagger$& 0.0 & 0.0 & 0.0 & 48.0 & 0.0 & 0.0 & 0.0 & 0.0 & 
0.0 & 0.0 & 0.0 & 0.0 & 0.0 & 16.0 & 8.0 & 4.0 & 52.0 & 60.0 & 10.4 \\
ViCLIP~\cite{wang2023internvid}  & 69.3 & \textbf{96.0} & 21.3 & 90.7 & \textbf{100.0} & 88.0 & \textbf{56.0} & \textbf{76.0} & 
14.7 & \textbf{100.0} & 42.7 & 78.7 & 80.0 & 77.3 & \textbf{38.7} & 69.3 & \textbf{100.0} & 89.3 & 71.6 \\
\hline
\rowcolor{mygray}
MN (Ours)  & \textbf{72.0} & \textbf{96.0} & \textbf{33.3} & \textbf{96.0} & \textbf{100.0} & \textbf{89.3} & \textbf{56.0} & \textbf{76.0} & 
\textbf{17.3} & \textbf{100.0} & \textbf{49.3} & \textbf{84.0} & \textbf{82.7} & \textbf{85.3} & 37.3 & \textbf{80.0} & \textbf{100.0} & \textbf{93.3} & \textbf{74.9} \\
\thickhline
\end{tabular}}
\caption{Performance comparison of action recognition across 18 diverse tasks. $\dagger$ indicates results obtained without task-specific fine-tuning (zero-shot).}
\label{tab:au}
    \vspace{-3mm}
\end{table*}

\subsection{Implementation details}
\label{exp_implementation}

\paragraph{Action Understanding} 
We adopt ViCLIP~\cite{wang2023internvid}, a generic video-language model, consisting of a video encoder (a standard ViT~\cite{dosovitskiy2020vit} with spatiotemporal attention) and a text encoder (identical to that of CLIP~\cite{pmlr-clip}). 
We employ the model weights pretrained on InternVid~\cite{wang2023internvid} and fine-tune it in the object manipulation domain.
Specifically, we use demonstration trajectories from RLBench~\cite{james2019rlbench,shridhar2022peract}, where RGB videos rendered from a front-facing perspective are paired with language instructions, to fine-tune both the video encoder and the text encoder. 
The learning rate of the video encoder is $1 \times 10^{-5}$, and that of the text encoder is 5\% of it.
Subsequently, we evaluate action recognition accuracy on a test set by selecting the language instruction with the highest similarity for each test video and computing the accuracy accordingly.

\paragraph{Embodied Execution} 
We use an Autoregressive Policy (ARP)~\cite{zhang2024arp} with a Multi-View Transformer (MVT)~\cite{goyal2023rvt,goyal2024rvt2} backbone. The model takes multi-view RGBD images (processed into point clouds) and language instructions as input, predicting the next target end-effector pose and gripper states.
We conduct experiments on a standard multi-task manipulation benchmark from RLBench~\cite{james2019rlbench}. The benchmark includes 18 tasks, each defined by a language description and featuring 2 to 60 variations, such as different object colors or locations. A Franka Panda robot with a parallel jaw gripper is tasked with execution, simulated via CoppeliaSim~\cite{coppeliaSim}. We train and test on the same dataset as prior works~\cite{shridhar2022peract,goyal2023rvt,goyal2024rvt2,zhang2024arp}, using 100 demonstrations per task for training and 25 unseen demonstrations for testing. Training settings and hyperparameters remain identical to those in the baseline~\cite{zhang2024arp}.

\paragraph{Representation Alignment}
We take the output features of the video encoder and the last block of the policy network and map them to $\mathbb{Z} \subset \mathbb{R}^{1024}$, using two separate linear layers. For contrastive learning training, we construct positive sample pairs based on language instruction consistency. That is, observing and executing actions with the same language instruction (\eg, \textit{``take the steak off the grill''}) are encouraged to align in representation, even if they do not come from the exact same episode (\eg, differing in object layout or action sequence). We set the temperature parameter to $\tau = 0.1$ and the learning rate to $1 \times 10^{-4}$. The loss weights are set as $\mathcal{L}_\text{EE} = 1$ and $\mathcal{L}_\text{align} = 0.5$.

\subsection{Action Understanding}

First, we investigate the effect of the proposed mirror neuron alignment framework on the action understanding model. To this end, we compare the top-1 action recognition accuracy with regard to language instruction among the following models: the original ViCLIP model pretrained on InternVid~\cite{wang2023internvid}, tested in a zero-shot manner; the same model fine-tuned on object interaction data; and one trained jointly with the proposed mirror neuron alignment objective. For the latter two, all other training conditions remain the same.

\Cref{tab:au} shows that ViCLIP exhibits some level of zero-shot action recognition capability, considering the total number of 200+ possible fine-grained variants. 
However, its performance on most action classes is suboptimal, as fine-grained classification requires nuanced spatial (\eg, \textit{``stack the wine bottle to the right of the rack''}), temporal (\eg, \textit{``push the maroon button, then push the green button''}), and quantitative reasoning (\eg, \textit{``stack 2 maroon blocks''}), which is challenging and relatively scarce in large-scale pretraining data.
Fine-tuning on the corresponding dataset significantly enhances the model performance.
Notably, our approach extensively outperforms the baseline in action recognition accuracy.

We attribute this improvement to the intrinsic connection and complementary nature between the embodied execution task and action recognition.  
For example, action recognition must distinguish between commands such as \emph{“put the ring on the azure spoke”} and \emph{“put the ring on the yellow spoke”}, recognizing their nuances in relation to the corresponding goal. Meanwhile, embodied execution not only requires identifying the correct 3D interaction locations (affordances) but also generating the appropriate motion trajectory to complete the action.  
The feature interaction and alignment mechanisms in the proposed framework facilitate learning a more comprehensive task representation, which in turn enhances the generalization ability of action recognition. These results demonstrate the synergistic role of embodied execution in improving action understanding, highlighting the effectiveness of our approach.

\begin{table*}[t]
\center
\small
\setlength{\tabcolsep}{1pt} %
\resizebox{1.0\linewidth}{!}{
\begin{tabular}{l | c c c c c c c c c c c c c c c c c c | c}
\thickhline 
Method & \makecell{Close\\Jar} & \makecell{Drag\\Stick} & \makecell{Insert\\Peg} & \makecell{Meat off\\Grill} & \makecell{Open\\Drawer} & \makecell{Place\\Cup} & \makecell{Place\\Wine} & \makecell{Push\\Btn.} & \makecell{Put in\\Cpd.} & 
\makecell{Put in\\Drawer} & \makecell{Put in\\Safe}  & \makecell{Screw\\Bulb} & \makecell{Slide\\Block} & \makecell{Sort\\Shape} & \makecell{Stack\\Block} & \makecell{Stack\\Cup} & \makecell{Sweep\\Dust} & \makecell{Turn\\Tap} & Avg. \\
\hline 
BC-Z CNN~\cite{jang2021bcz} & 0.0 & 0.0 & 0.0 & 0.0 & 4.0 & 0.0 & 0.0 & 0.0 & 0.0 & 8.0 & 4.0 & 0.0 & 0.0 & 0.0 & 0.0 & 0.0 & 0.0 & 8.0 & 1.3 \\
BC-Z ViT~\cite{jang2021bcz} & 0.0 & 0.0 & 0.0 & 0.0 & 0.0 & 0.0 & 0.0 & 0.0 & 0.0 & 0.0 & 0.0 & 0.0 & 0.0 & 0.0 & 0.0 & 0.0 & 0.0 & 16.0 & 1.3 \\
C2F-ARM~\cite{James_2022_CVPR} & 24.0 & 24.0 & 4.0 & 20.0 & 20.0 & 0.0 & 8.0 & 72.0 & 0.0 & 4.0 & 12.0 & 8.0 & 16.0 & 8.0 & 0.0 & 0.0 & 0.0 & 68.0 & 20.1 \\
HiveFormer~\cite{guhur2022instructiondriven} &  52.0 & 76.0 & 0.0 & 100.0 & 52.0 & 0.0 & 80.0 & 84.0 & 32.0 & 68.0 & 76.0 & 8.0 & 64.0 & 8.0 & 8.0 & 0.0 & 28.0 & 80.0 & 45.3 \\
PolarNet~\cite{chen23polarnet} & 36.0 & 92.0 & 4.0 & 100.0 & 84.0 & 0.0 & 40.0 & 96.0 & 12.0 & 32.0 & 84.0 & 44.0 & 56.0 & 12.0 & 4.0 & 8.0 & 52.0 & 80.0 & 46.4 \\
PerAct~\cite{shridhar2022peract} & 55.2 \stddev{4.7} & 89.6 \stddev{4.1} & 5.6 \stddev{4.1} & 70.4 \stddev{2.0} & 88.0 \stddev{5.7} & 2.4 \stddev{3.2} & 44.8 \stddev{7.8} & 92.8 \stddev{3.0} & 28.0 \stddev{4.4} & 51.2 \stddev{4.7} & 84.0 \stddev{3.6} & 17.6 \stddev{2.0} & 74.0 \stddev{13.0} & 16.8 \stddev{4.7} & 26.4 \stddev{3.2} & 2.4 \stddev{2.0} & 52.0 \stddev{0.0} & 88.0 \stddev{4.4} & 49.4 \\
Act3D~\cite{gervet2023actd} & 92.0 & 92.0 & 27.0 & 94.0 & 93.0 & 3.0 & 80.0 & 99.0 & 51.0 & 90.0 & 95.0 & 47.0 & 93.0 & 8.0 & 12.0 & 9.0 & 92.0 & 94.0 & 65.0 \\
RVT~\cite{goyal2023rvt} & 52.0 \stddev{2.5} & 99.2 \stddev{1.6} & 11.2 \stddev{3.0} & 88.0 \stddev{2.5} & 71.2 \stddev{6.9} & 4.0 \stddev{2.5} & 91.0 \stddev{5.2} & 100.0 \stddev{0.0} & 49.6 \stddev{3.2} & 88.0 \stddev{5.7} & 91.2 \stddev{3.0} & 48.0 \stddev{5.7} & 81.6 \stddev{5.4} & 36.0 \stddev{2.5} & 28.8 \stddev{3.9} & 26.4 \stddev{8.2} & 72.0 \stddev{0.0} & 93.6 \stddev{4.1} & 62.9 \\
RVT-2~\cite{goyal2024rvt2} & \textbf{100.0}\stddev{0.0} & 99.0\stddev{1.7} & 40.0\stddev{0.0} & \textbf{99.0}\stddev{1.7} & 74.0\stddev{11.8} & 38.0\stddev{4.5} & 95.0\stddev{3.3} & \textbf{100.0}\stddev{0.0} & 66.0\stddev{4.5} & 96.0\stddev{0.0} & \textbf{96.0}\stddev{2.8} & 88.0\stddev{4.9} & \textbf{92.0}\stddev{2.8} & 35.0\stddev{7.1} & \textbf{80.0}\stddev{2.8} & 69.0\stddev{5.9} & \textbf{100.0}\stddev{0.0} & 99.0\stddev{1.7} & 81.4 \\
ARP~\cite{zhang2024arp} &  \textbf{100.0}\stddev{0.0} & \textbf{100.0}\stddev{0.0} & 93.3\stddev{2.3} & 92.0\stddev{0.0} & 90.7\stddev{2.3} & 49.3\stddev{6.1} & 93.3\stddev{8.3} & \textbf{100.0}\stddev{0.0} & 66.7\stddev{6.1} & \textbf{100.0}\stddev{0.0} & 88.0\stddev{4.0} & \textbf{92.0}\stddev{4.0} & 86.7\stddev{6.1} & 49.3\stddev{6.1} & 56.0\stddev{4.0} & 82.7\stddev{6.1} & 98.7\stddev{2.3} & 97.3\stddev{2.3} & 85.3 \\
\hline
\rowcolor{mygray}
MN (Ours) & \textbf{100.0}\stddev{0.0} & \textbf{100.0}\stddev{0.0} & \textbf{94.7}\stddev{4.6} & 93.3\stddev{2.3} & \textbf{93.3}\stddev{2.3} & \textbf{53.3}\stddev{12.2} & \textbf{97.3}\stddev{2.3} & \textbf{100.0}\stddev{0.0} & \textbf{70.7}\stddev{2.3} & \textbf{100.0}\stddev{0.0} & 93.3\stddev{2.3} & 88.0\stddev{0.0} & 82.7\stddev{2.3} & \textbf{66.7}\stddev{10.1} & 72.0\stddev{6.9} & \textbf{93.3}\stddev{4.6} & \textbf{100.0}\stddev{0.0} & \textbf{100.0}\stddev{0.0} & \textbf{88.8} \\
\thickhline
\end{tabular}}
\caption{Performance comparison of embodied execution across 18 diverse tasks. A single model is evaluated across all tasks, with success criteria defined according to RLBench~\cite{james2019rlbench}.
}
\label{tab:ee}
\end{table*}

\begin{figure*}[t]
    \centering
    \includegraphics[width=\linewidth]{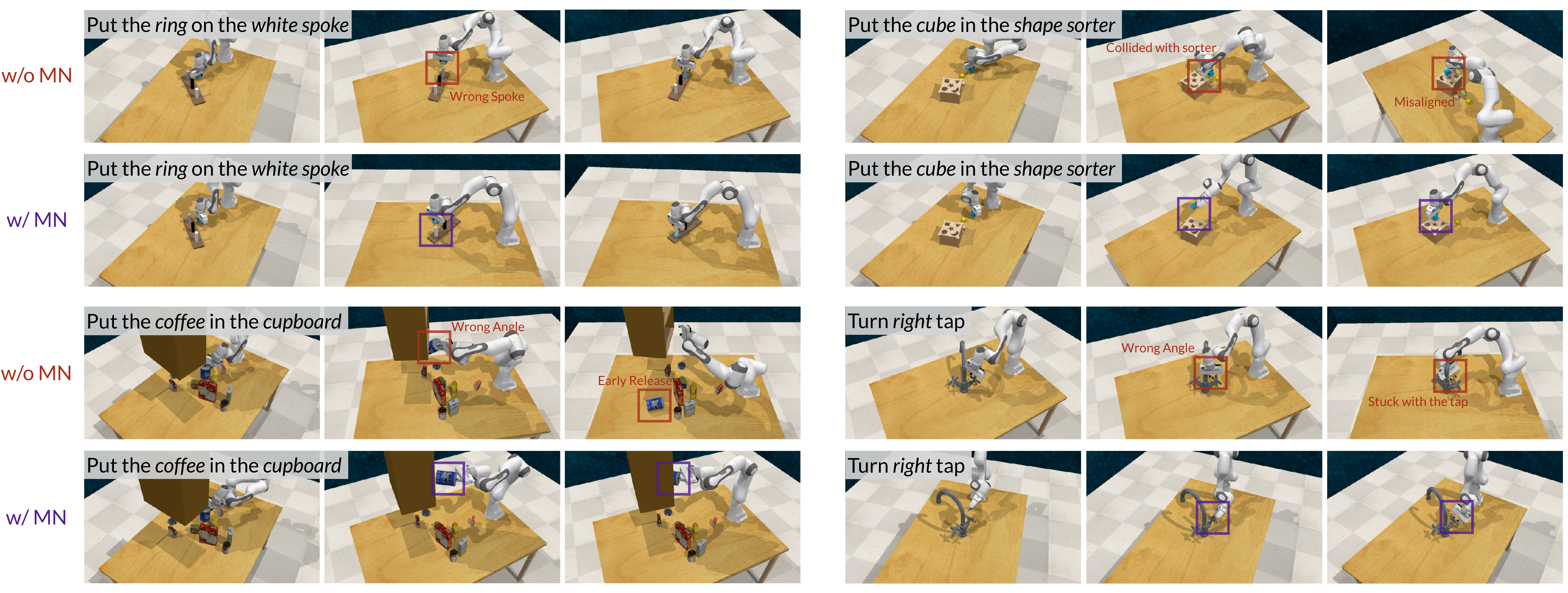}
    \caption{Visualization comparison of embodied execution results. We compare our approach with a baseline model without the MN design (\ie, ARP~\cite{zhang2024arp}). The visualization sequences are captured from rotating camera perspectives, with key details highlighted in boxes for clarity. The proposed MN design helps learning affordances and fine-grained operations, leading to improved performance.}
    \label{fig:visual_compare}
        \vspace{-3mm}
\end{figure*}

\subsection{Embodied Execution}

We investigate the impact of the proposed framework on embodied execution tasks, specifically language-conditioned multi-task object manipulation. Following prior works, we evaluate the trained policy by performing rollouts in unseen environments from RLBench~\cite{james2019rlbench} and measuring the success rates for various tasks.
We compare our approach with various baselines, including simple image-to-action behavioral cloning baselines~\cite{jang2021bcz}, 3D object manipulation methods~\cite{james2019rlbench,shridhar2022peract,guhur2022instructiondriven,goyal2023rvt,goyal2024rvt2}, as well as the most directly comparable baseline, ARP~\cite{zhang2024arp}. 
All approaches are trained and tested using input images of size $128 \times 128$, except for Act3D, which uses images of size $256 \times 256$.

As shown in~\Cref{tab:ee}, our method achieves notable improvements over state-of-the-art approaches in most tasks, with an average success rate increase of 3.5\% compared to the direct baseline. Furthermore, it demonstrates significant gains in tasks requiring fine-grained affordance reasoning, such as \textit{Sort Shape} and \textit{Stack Cup}.  
To better understand the reasons behind the improvements, we compare the rollout trajectories of models with the MN design against those without MN (which degenerates to ARP~\cite{zhang2024arp}). 
\Cref{fig:visual_compare} presents several comparison results. Our proposed MN method demonstrates enhanced understanding of affordances, such as how to interact with different objects. It also improves target accuracy and refines the precision of fine-grained interactions. Additional details can be found in the supplementary videos.

Several factors may contribute to these improvements. First, the representation alignment with the action recognition model enables learning a more comprehensive and disentangled task representation, which helps improve robust generalization. Second, action recognition may provide higher-quality appearance and geometric representations, which, through alignment, enhance the representation learning of embodied execution and improve its corresponding task performance.

\begin{figure*}[t]
    \centering
    \includegraphics[width=\linewidth]{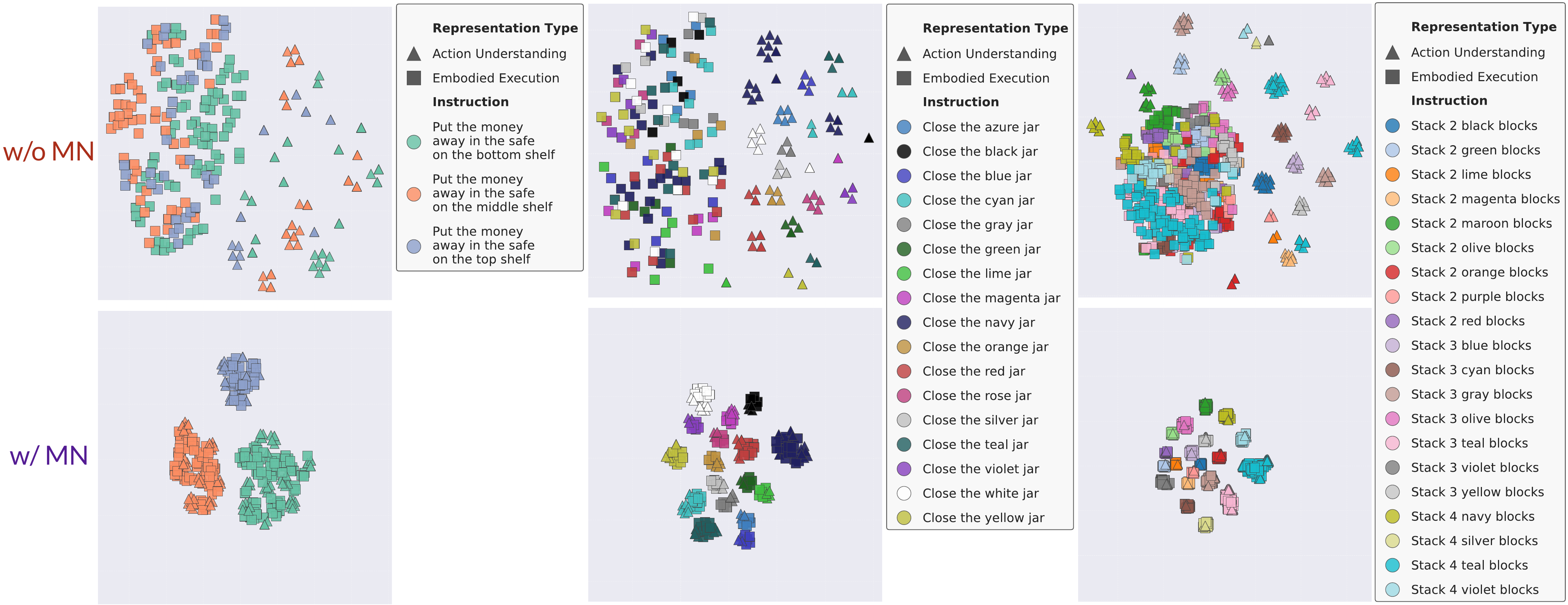}
    \caption{Visualization comparison of latent representations. We align the representations of two models with the MN design and two models without it into their respective shared latent spaces, followed by dimensionality reduction for visualization. Different shapes represent representations from different models, while different colors indicate the corresponding language instructions.}
    \label{fig:reps}
        \vspace{-3mm}
\end{figure*}

\subsection{Representation Analysis}

We aim to further analyze the impact of the proposed MN module on learned neural representations to gain deeper insight.
We align the representations of action understanding and embodied execution to a shared latent space \(\mathbb{Z}\) using alignment probing and then visualize them via t-SNE~\cite{JMLR:v9:vandermaaten08a} dimensionality reduction. Specifically, we compare baseline models that are trained separately without the MN module to our proposed MN approach, which couples the training of both models.

\Cref{fig:reps} shows that our proposed method not only facilitates the alignment of representations between action understanding and embodied execution, as indicated by shapes of the same color clustering together, but also enhances the ability to distinguish fine-grained nuances in instructions, with different colors forming distinct clusters.
We believe that this representation alignment and disentanglement may be the cause of the improved generalization ability.
Additionally, we find that the MN module exhibits a similar disentanglement effect on representations before the linear transformation (in higher-dimensional space). We also illustrate the evolution of representations throughout the training process; please see \cref{sec:supp_results} for details.

\subsection{Ablation Study}

Additionally, we study how the implementation of the proposed mirror neuron alignment module influences model performance.
We primarily explore two questions: (1) What kind of action representations should be aligned? (2) How strict should the alignment criterion be? These questions correspond to fundamental aspects of contrastive learning: positive sample construction and temperature.

In our experiments, we explore three different strategies for constructing positive samples:

\begin{enumerate}

    \item \textbf{By Episode:} Positive samples are drawn from the same episode. This means that the action understanding and embodied execution models are aligned when learning from the exact same episode, including object placement, initial positions, and action progression, although their input and prediction modalities are different. In practice, we sample paired data from the same episode for both models to construct positive sample pairs.

    \item \textbf{By Instruction:} Positive samples need not originate from the same episode. Even if scene layouts and action sequences differ, as long as the underlying goal corresponds to the same language instruction (\eg, \textit{``open the top drawer''}), their representations are aligned. In implementation, for each sample, we randomly sample a paired positive sample from the subset corresponding to the same instruction. This is the default setting used in our other experiments.

    \item \textbf{By Task:} This is the most relaxed criterion. Positive samples do not even need to correspond to the same language instruction (\eg, \text{``open the top drawer''} and \textit{``open the bottom drawer''}). As long as they belong to the same action class (\textit{Open Drawer}), we align their representations accordingly.
    
\end{enumerate}

\begin{table}[t]
    \centering
    \footnotesize
    \setlength{\tabcolsep}{7.5pt} %
    \begin{tabular}{l | c c c | c c c}
    \thickhline 
    Task & \multicolumn{3}{c|}{Action Understanding} & \multicolumn{3}{c}{Embodied Execution} \\
    \hline
    Temp. ($\tau$) & 0.02 & 0.1 & 0.2 & 0.02 & 0.1 & 0.2 \\
    \hline
    By Episode  & 74.0 & 72.9 & 74.0 & 86.4 & 88.1 & 87.8 \\
    By Instruction  & 74.9 & 74.9 & 77.1 & 86.7 & 88.8 & 87.0 \\
    By Class & 73.6 & 71.6 & 72.2 & 86.9 & 85.7 & 85.6 \\
    \thickhline
    \end{tabular}
    \caption{Ablation study results. We analyze the impact of alignment strategies and contrastive learning temperatures ($\tau$). Model performance is evaluated on both Action Understanding and Embodied Execution tasks.}
    \label{tab:ablation}
    \vspace{-4mm}
\end{table}

Since the choice of a reasonable temperature parameter is closely related to the definition of positive samples, we evaluate each strategy using three different temperature values: $\tau = $ 0.02, 0.1, 0.2. Based on the results in~\Cref{{tab:ablation}}, we derive the following observations:

\begin{itemize}
    \item All strategies generally outperform the baseline, demonstrating the effectiveness of aligning representations of the two modules akin to MN.
    
    \item The optimal temperature varies depending on the positive sample construction strategy. For example, \textit{by episode} benefits from a higher temperature, whereas \textit{by class} performs better with a lower temperature. This aligns with the relationship between positive sample definitions and intra-/inter-class distances.

    \item Overall, \textit{by instruction} proves to be a well-balanced strategy, maintaining a good trade-off between variation and semantic consistency, leading to the best generalization performance. In contrast, \textit{by class} exhibits slightly lower performance, possibly because the two tasks still require distinguishing between fine-grained language descriptions.
\end{itemize}
\section{Related Work}

\subsection{Mirror Neurons}

The discovery of mirror neurons (MNs) in macaques represents a pivotal advancement in understanding the neural mechanisms underlying action perception and execution. These neurons, primarily located in area F5 of the premotor cortex and in the inferior parietal lobule (IPL), activates both when an individual performs a specific action and when they observe another individual executing the same or a similar action~\citep{di1992understanding,gallese1996action,rizzolatti2004mirror}.
In humans, neuroimaging studies provide evidence for a homologous mirror neuron system, including motor-related regions such as the precentral gyrus and the inferior frontal gyrus~\citep{gazzola2009observation,iacoboni1999cortical,keysers2009expanding}. These areas are activated during both action execution and observation, suggesting their role in forming shared neural representations of motor and sensory experiences. This shared representation is thought to facilitate action understanding, imitation, and potentially social cognition~\citep{umilta2001know,iacoboni2005grasping}.
Inspired by the functionality of MNs, several studies propose corresponding computational models~\citep{oztop2002mirror,zhong2011robot,liu2019mirroring,seker2022imitation}. Nonetheless, existing approaches have yet to fully establish a unified representation that integrates embodied action execution with action understanding. Furthermore, the contribution of mirroring to these cognitive processes remains underexplored.

\subsection{Representation Alignment}

Representation alignment seeks to bridge the gap between feature representations across different models, modalities, or domains, emerging as a vital topic in machine learning. It has been shown to serve as a meaningful objective for improving model performance, enhancing training efficiency, and enabling generalization across diverse tasks. For instance, prior studies have explored the similarity of neural network representations, revealing that models with different architectures and initializations tend to exhibit a certain degree of alignment~\cite{morcos2018insights,kornblith2019similarity}.
In the context of cross-modal alignment, researchers achieve robust task generalization by aligning vision and language representations~\cite{pmlr-clip}, scaling this process with noisy text supervision~\cite{jia2021scaling}, or leveraging pretrained unimodal models~\cite{zhang2024assessing}. Moreover, recent studies have found that internal representations tend to align even within separately trained unimodal models~\cite{maniparambil2024vision,huh2024position,zhang2024assessing}. Specifically, \textit{The Platonic Representation Hypothesis}~\cite{huh2024position} suggests that this convergence reflects the emergence of a shared statistical model of reality, reminiscent of Plato’s concept of an idealized world.
In generative modeling, aligning with pretrained discriminative representations has also been shown to significantly enhance image generation quality and accelerate training convergence~\cite{yu2025representation,ma2024janusflow}.
Our approach explores representation alignment from a novel perspective, integrating insights from mirror neurons. Specifically, we investigate how neural representations for action understanding and embodied execution align at both the task and functional levels in embodied agents.

\section{Conclusion}
In this paper, we present a novel framework that unifies action understanding and embodied execution through representation learning, inspired by the biological mechanism of mirror neurons. 
We first discover that models separately trained for these two tasks exhibit spontaneous representation alignment, which is associated with task success.
Building on this insight, we introduce an approach that explicitly aligns the representations of observed and executed actions within a shared latent space using contrastive learning. Experiments on action recognition and multi-task object manipulation benchmarks show that this simple method promotes synergy between the two tasks, enhancing representation quality and generalization.

We hope our work offers a novel perspective by treating action understanding and embodied execution as intertwined, rather than modeling them as isolated cognitive processes. On a broader scale, it reflects how cognition emerges from sensorimotor engagement with the environment, as emphasized by embodied cognition.
Future work could also adopt more sophisticated representation learning strategies, such as hierarchical alignment, and incorporate multisensory integration to better handle complex real-world tasks. Finally, exploring aspects of social cognition could further enrich the framework by capturing interactive and cooperative dynamics.

\clearpage
\section*{Acknowledgment}
This work was supported by National Science and Technology Major Project (2022ZD0114904) and NSFC-6247070125.

{
    \small
    \bibliographystyle{ieeenat_fullname}
    \bibliography{main}
}

\clearpage
\setcounter{page}{1}
\maketitlesupplementary

\appendix

\begin{table*}[h]
\centering
\scriptsize
\begin{tabular}{llccl} 
\toprule
Task                      & Variation   Type           & \# of Variations     & Avg. Keyframes       & Language     Template         \\
\midrule
\texttt{open drawer}      & placement                  &           3          &           3.0        & ``open the \blank drawer'' \\
\texttt{slide block}      & color                      &           4          &           4.7        & ``slide the block to \blank target'' \\
\texttt{sweep to dustpan} & size                       &           2          &           4.6        & ``sweep dirt to the \blank dustpan'' \\
\texttt{meat off grill}   & category                   &           2          &           5.0        & ``take the \blank off the grill'' \\
\texttt{turn tap}         & placement                  &           2          &           2.0        & ``turn \blank tap'' \\
\texttt{put in drawer}    & placement                  &           3          &          12.0        & ``put the item in the \blank drawer'' \\
\texttt{close jar}        & color                      &          20          &           6.0        & ``close the \blank jar'' \\
\texttt{drag stick}       & color                      &          20          &           6.0        & ``use the stick to drag the cube onto the \blank target'' \\
\texttt{stack blocks}     & color, count               &          60          &          14.6        & ``stack \blank \blank blocks''  \\
\texttt{screw bulb}       & color                      &          20          &           7.0        & ``screw in the \blank light bulb'' \\
\texttt{put in safe}      & placement                  &           3          &           5.0        & ``put the money away in the safe on the \blank shelf'' \\
\texttt{place wine}       & placement                  &           3          &           5.0        & ``stack the wine bottle to the \blank of the rack'' \\
\texttt{put in cupboard}  & category                   &           9          &           5.0        & ``put the \blank in the cupboard'' \\
\texttt{sort shape}       & shape                      &           5          &           5.0        & ``put the \blank in the shape sorter'' \\
\texttt{push buttons}     & color                      &          50          &           3.8        & ``push the \blank button, [then the \blank button]'' \\
\texttt{insert peg}       & color                      &          20          &           5.0        & ``put the ring on the \blank spoke'' \\
\texttt{stack cups}       & color                      &          20          &          10.0        & ``stack the other cups on top of the \blank cup'' \\
\texttt{place cups}       & count                      &           3          &          11.5        & ``place \blank cups on the cup holder'' \\
\bottomrule
\end{tabular}
\caption{Language-conditioned tasks and variations in RLBench~\citep{james2019rlbench}.}
\label{table:task_desc}
\end{table*}

\section{Theoretical Analysis}
\label{sec:theoretical}

\begin{theorem}
    \label{thm_probe}
    The mutual information between the action understanding representation \bmu and the embodied execution representation \bme can be estimated by optimizing the transformations \( \mathcal{T}_u \) and \( \mathcal{T}_e \) to minimize the bidirectional alignment loss \( \mathcal{L}_{\text{align}} \).
\end{theorem}

\begin{proof}
    The mutual information between \( \bm{u} \) and \( \bm{e} \) is defined as:
    \begin{equation}
        I(\bm{u}; \bm{e}) = D_{\mathrm{KL}}(p(\bm{u}, \bm{e}) \| p(\bm{u}) p(\bm{e})).
    \end{equation}
    Since direct computation is intractable, we introduce trainable transformations:
    \begin{equation}
        \bm{z}_u = \mathcal{T}_u(\bm{u}), \quad \bm{z}_e = \mathcal{T}_e(\bm{e}),
    \end{equation}
    where \( \mathcal{T}_u \) and \( \mathcal{T}_e \) are optimized via a loss function.
    By the \textit{Data Processing Inequality (DPI)}, these transformations satisfy:
    \begin{equation}
        I(\bm{z}_u; \bm{z}_e) \leq I(\bm{u}; \bm{e}),
    \end{equation}
with equality if \( \mathcal{T}_u \) and \( \mathcal{T}_e \) preserve all relevant information. Thus, we estimate \( I(\bm{u}; \bm{e}) \) indirectly via \( I(\bm{z}_u; \bm{z}_e) \).

To estimate \( I(\bm{z}_u; \bm{z}_e) \), we approximate the conditional distributions using contrastive learning. For a batch of size \( B \), define:
    \begin{equation}
        \hat{p}(\bm{z}_u | \bm{z}_e) = \frac{\exp(\text{sim}(\bm{z}_u, \bm{z}_e) / \tau)}{\sum_{j=1}^{B} \exp(\text{sim}(\bm{z}_u, \bm{z}_e^{(j)}) / \tau)},
    \end{equation}
    where \( \bm{z}_e^{(j)} \) are batch samples, \( \text{sim} \) is a similarity function (\eg, cosine similarity), and \( \tau \) is a temperature parameter. This approximates the intractable sum over \( p(\bm{z}_e) \).

    Since KL divergence is non-negative:
    \begin{equation}
        D_{\mathrm{KL}}(p(\bm{z}_u | \bm{z}_e) \| \hat{p}(\bm{z}_u | \bm{z}_e)) \geq 0,
    \end{equation}
    it follows that:
    \begin{equation}
        \mathbb{E}_{p(\bm{z}_u, \bm{z}_e)} [\log p(\bm{z}_u | \bm{z}_e)] \geq \mathbb{E}_{p(\bm{z}_u, \bm{z}_e)} [\log \hat{p}(\bm{z}_u | \bm{z}_e)].
    \end{equation}

    Similarly, for the reverse direction:
    \begin{equation}
        \hat{p}(\bm{z}_e | \bm{z}_u) = \frac{\exp(\text{sim}(\bm{z}_e, \bm{z}_u) / \tau)}{\sum_{j=1}^{B} \exp(\text{sim}(\bm{z}_e, \bm{z}_u^{(j)}) / \tau)},
    \end{equation}
    and:
    \begin{equation}
        \mathbb{E}_{p(\bm{z}_u, \bm{z}_e)} [\log p(\bm{z}_e | \bm{z}_u)] \geq \mathbb{E}_{p(\bm{z}_u, \bm{z}_e)} [\log \hat{p}(\bm{z}_e | \bm{z}_u)].
    \end{equation}

    The mutual information \( I(\bm{z}_u; \bm{z}_e) \) can be expressed as:
    \begin{equation}
        I(\bm{z}_u; \bm{z}_e) = \mathbb{E}_{p(\bm{z}_u, \bm{z}_e)} [\log p(\bm{z}_u | \bm{z}_e)] - \mathbb{E}_{p(\bm{z}_u)} [\log p(\bm{z}_u)].
    \end{equation}
    Assuming the batch approximates \( p(\bm{z}_u) \) as uniform (a common heuristic in contrastive learning):
    \begin{equation}
        \mathbb{E}_{p(\bm{z}_u)} [\log p(\bm{z}_u)] \approx -\log B.
    \end{equation}

Define the single-direction InfoNCE loss:
\begin{equation}
    \mathcal{L}_{\text{InfoNCE}} = - \frac{1}{B} \sum_{i=1}^{B} \log \frac{\exp(\text{sim}(\bm{z}_u^{(i)}, \bm{z}_e^{(i)}) / \tau)}{\sum_{j=1}^{B} \exp(\text{sim}(\bm{z}_u^{(i)}, \bm{z}_e^{(j)}) / \tau)}.
\end{equation}
Substituting into the mutual information bound:
\begin{equation}
    I(\bm{z}_u; \bm{z}_e) \geq \log B - \mathcal{L}_{\text{InfoNCE}}.
\end{equation}

Given the bidirectional alignment loss defined earlier in Equation~\ref{eq:align_loss}, we note that:
\begin{equation}
    \mathcal{L}_{\text{align}} = \frac{1}{2} \left( \mathcal{L}_{ u \to e} + \mathcal{L}_{e \to u} \right),
\end{equation}
where \( \mathcal{L}_{u \to e} \) is the InfoNCE loss from \( \bm{z}_u \) to \( \bm{z}_e \), and \( \mathcal{L}_{e \to u} \) is from \( \bm{z}_e \) to \( \bm{z}_u \). Since each provides a bound:
\begin{equation}
    I(\bm{z}_u; \bm{z}_e) \geq \log B - \mathcal{L}_{u \to e}, \quad I(\bm{z}_u; \bm{z}_e) \geq \log B - \mathcal{L}_{e \to u},
\end{equation}
substituting \( \mathcal{L}_{\text{align}} \), the combined lower bound becomes:
\begin{equation}
    I(\bm{z}_u; \bm{z}_e) \geq \log B - \mathcal{L}_{\text{align}}.
\end{equation}

By optimizing \( \mathcal{T}_u \) and \( \mathcal{T}_e \) to minimize \( \mathcal{L}_{\text{align}} \), the lower bound \( \log B - \mathcal{L}_{\text{align}} \) is maximized. If \( \mathcal{T}_u \) and \( \mathcal{T}_e \) preserve sufficient information, \( I(\bm{z}_u; \bm{z}_e) \) approximates \( I(\bm{u}; \bm{e}) \), providing an indirect estimate of \( I(\bm{u}; \bm{e}) \). This completes the proof.
\end{proof}

\begin{theorem}
    The mutual information \( I(\bm{u}; \bm{e}) \) between the action understanding representation \( \bm{u} \) generated by the model \aup and the embodied execution representation \( \bm{e} \) generated by the model \eep can be maximized by simultaneously optimizing \aup and \eep, along with linear transformations \tu and \te, to minimize the bidirectional alignment loss \( \mathcal{L}_{\text{align}} \), provided \tu and \te preserve sufficient information.
    \label{thm_align}
\end{theorem}

\begin{proof}
 Let \( \bm{u} \) and \( \bm{e} \) be representations generated by the action understanding model \( \mathcal{U} \) and the embodied execution model \( \mathcal{E} \), parameterized by \( \theta_u \) and \( \theta_e \), respectively. Define linear transformations:
 \begin{equation}
 \bm{z}_u = \mathcal{T}_u(\bm{u}; \phi_u), \quad \bm{z}_e = \mathcal{T}_e(\bm{e}; \phi_e),
 \end{equation}
 where \( \phi_u \) and \( \phi_e \) are the parameters of \( \mathcal{T}_u \) and \( \mathcal{T}_e \). The mutual information satisfies:
 \begin{equation}
 I(\bm{z}_u; \bm{z}_e) \leq I(\bm{u}; \bm{e}),
 \end{equation}
 with equality if \( \mathcal{T}_u \) and \( \mathcal{T}_e \) are invertible.

 Based on the bidirectional alignment loss as defined in \cref{eq:align_loss}, we optimize the parameters:
 \begin{equation}
 \{\theta_u^*, \theta_e^*, \phi_u^*, \phi_e^*\} = \arg\min_{\theta_u, \theta_e, \phi_u, \phi_e} \mathcal{L}_{\text{align}}(\theta_u, \theta_e, \phi_u, \phi_e),
 \end{equation}
 augmenting the original objectives of \( \mathcal{U} \) and \( \mathcal{E} \). This minimizes \( \mathcal{L}_{\text{align}} \), maximizing:
 \begin{equation}
 I(\bm{z}_u; \bm{z}_e) \geq \ log B - \mathcal{L}_{\text{align}}(\theta_u^*, \theta_e^*, \phi_u^*, \phi_e^*).
 \end{equation}
     Optimizing \( \theta_u \) and \( \theta_e \) adjusts \( \bm{u} \) and \( \bm{e} \) to increase \( I(\bm{u}; \bm{e}) \), while optimizing \( \phi_u \) and \( \phi_e \) aligns \( \bm{z}_u \) and \( \bm{z}_e \) with the constraint \( I(\bm{z}_u; \bm{z}_e) \leq I(\bm{u}; \bm{e}) \). Assuming \( \mathcal{T}_u \) and \( \mathcal{T}_e \) preserve sufficient information, joint optimization reduces information loss between \( \bm{u}, \bm{e} \) and \( \bm{z}_u, \bm{z}_e \), allowing \( I(\bm{z}_u; \bm{z}_e) \) to closely approximate \( I(\bm{u}; \bm{e}) \). Hence, maximizing \( I(\bm{z}_u; \bm{z}_e) \) through \( \mathcal{L}_{\text{align}} \) also maximizes \( I(\bm{u}; \bm{e}) \), completing the proof.
\end{proof}

\section{Environment Details}

 \paragraph{Tasks} Our \textit{action recognition} and \textit{embodied execution} tasks follow the multi-task definition from previous work~\cite{shridhar2022peract,goyal2023rvt,goyal2024rvt2,zhang2024arp} based on RLBench~\cite{james2019rlbench}. Specifically, there are 18 tasks with 249 variations, defined through diverse language instructions. These tasks include non-prehensile actions such as \textit{push buttons}, common pick-and-place tasks like \textit{place wine}, and high-precision peg-in-hole tasks such as \textit{insert peg}. \Cref{table:task_desc} provides an overview of these tasks.
\paragraph{Variations} Task variations include randomly sampled colors, sizes, shapes, counts, placements, and categories of objects. The set of colors include 20 instances: \texttt{colors} = $\{$\texttt{red}, \texttt{maroon}, \texttt{lime}, \texttt{green}, \texttt{blue}, \texttt{navy}, \texttt{yellow}, \texttt{cyan}, \texttt{magenta}, \texttt{silver}, \texttt{gray}, \texttt{orange}, \texttt{olive}, \texttt{purple}, \texttt{teal}, \texttt{azure}, \texttt{violet}, \texttt{rose}, \texttt{black}, \texttt{white}$\}$. The set of sizes include 2 instances: \texttt{sizes} = $\{$\texttt{short}, \texttt{tall}$\}$. The set of shapes include 5 instances: \texttt{shapes} = $\{$\texttt{cube}, \texttt{cylinder}, \texttt{triangle}, \texttt{star}, \texttt{moon}$\}$. The set of counts include 3 instances: \texttt{counts} = $\{$\texttt{1}, \texttt{2}, \texttt{3}$\}$. The placements and object categories are specific to each task. For instance, \texttt{open drawer} has 3 placement locations: \texttt{top}, \texttt{middle}, and \texttt{bottom}, and \texttt{put in cupboard} includes 9 YCB objects. In addition to these semantic variations, objects are placed on the tabletop at random poses. Some large objects like drawers have constrained pose variations~\citep{james2019rlbench} to ensure that manipulating them is kinematically feasible with the Franka arm.

\section{Implementation Details}
\label{supp_implementation}
\subsection{Action Recognition}

We follow ViCLIP~\cite{wang2023internvid} to implement the action recognition module. Specifically, the video encoder uses a standard ViT with spatiotemporal attention~\cite{dosovitskiy2020vit}. Random patch masking is applied to the input videos during pretraining, which significantly alleviates the computational burden. 
We use the model weights pretrained on InternVid~\cite{wang2023internvid} and fine-tune on video-text pairs of object interactions simulated in RLBench~\cite{james2019rlbench}. The training objective is to align the corresponding video and text embeddings, similar to CLIP~\cite{pmlr-clip}, using contrastive learning with a temperature parameter $\tau_{\text{viclip}} = 0.05$.

For action recognition evaluation, given an input video, we compute its video embedding and compare it with the text embeddings of all possible action classes using cosine similarity. The class with the highest similarity score is selected as the predicted label.

\subsection{Embodied Execution}

 We follow ARP~\cite{zhang2024arp} to implement the action recognition module. The experimental settings are consistent with prior works~\cite{shridhar2022peract,goyal2023rvt,goyal2024rvt2,zhang2024arp}. 
The input RGB-D images have a resolution of $128 \times 128$ and are captured by four noiseless cameras mounted at the front, left shoulder, right shoulder, and wrist of the robot. 

We use the next key end-effector pose as the control interface, eliminating the need for high-frequency actions. Consequently, neither the horizon nor action steps are applicable. Instead, low-level robot movements are generated using RLBench's built-in RRT planner. We use a chunk size of 2 for binary gripper states and a chunk size of 1 for end-effector positions and rotations. For example, ARP first predicts the roll, followed by the pitch and yaw of the rotation Euler angles. 
Following the strategy of RVT-2~\cite{goyal2024rvt2}, we first predict coarse positions and then refine them by zooming into the images (with updated vision features) to obtain more accurate positions. The end-effector positions are initially predicted in 2D, and the corresponding 3D positions are derived from the 2D coordinates in each viewpoint. \Cref{tab:hp-rlb} presents the training parameters.

 \begin{table}[t]\centering
\caption{Hyperparameters used for the embodied execution module on RLBench.}\label{tab:hp-rlb}
\begin{tabular}{lr}\toprule
\textbf{Hyperparameter} &\textbf{Value} \\\midrule
\multicolumn{2}{l}{\textit{Model}} \\\midrule
number of layers &8 \\
embedding size &128 \\
mlp size &512 \\
backbone &MVT~\cite{goyal2023rvt} \\\midrule
\textit{Action Sequence} & \\\midrule
chunk size &mix of 2 and 1 \\\midrule
\textit{Train \& Eval} & \\\midrule
observation &RGBD $4\times 128\times 128 \times 4$ \\
maximum evaluation steps &25 \\
train iterations &80000 \\
eval frequency &10000 \\
batch size &96*2 \\
learning rate &1.25e-5 \\
learning rate scheduler &cosine \\
optimizer &LAMB \\
\bottomrule
\end{tabular}
\end{table}

\subsection{Joint Training}
We perform end-to-end joint training of both tasks and representation alignment, as previously discussed. Since action recognition is easier to learn than embodied execution, we control the learning frequency of action recognition to 20\% to balance the training pace. 
We use a batch size of 192 and train for 25 hours on 2 NVIDIA A100 80GB GPUs.

\end{document}